\documentclass[10pt]{article}
\usepackage{iclr2026_conference,times}
\usepackage[utf8]{inputenc}
\usepackage[T1]{fontenc}
\usepackage{amsmath,amssymb,amsthm,mathtools}
\usepackage{graphicx}
\usepackage{booktabs}
\usepackage{array}
\usepackage{multirow}
\usepackage{enumitem}
\usepackage{listings}
\usepackage{float}
\usepackage{xcolor}
\usepackage{url}
\graphicspath{{figures/}}

\iclrfinalcopy
\title{Signed Compression Progress on a Sealed Audit\\
is Goodhart-Resistant}
\author{Ayush Mittal \qquad Dhruv Gupta}

\newtheorem{theorem}{Theorem}
\newtheorem{proposition}{Proposition}
\newtheorem{corollary}{Corollary}
\newtheorem{definition}{Definition}

\newcommand{\E}{\mathcal E}
\newcommand{\F}{\mathcal F}

\newcommand{\Q}{\mathsf Q}
\newcommand{\Pstream}{\mathsf P}
\newcommand{\R}{\mathbb R}

\newcommand{\hE}{\widehat{\mathcal E}}
\newcommand{\CP}{\mathrm{CP}}

\newcommand{\lean}[1]{\texttt{\def\_{\textunderscore\allowbreak}#1}}

\lstdefinestyle{leanstyle}{
  basicstyle=\ttfamily\footnotesize,
  breaklines=true,
  columns=fullflexible,
  frame=single,
  backgroundcolor=\color{black!2},
  xleftmargin=0.5em,
  xrightmargin=0.5em
}

\begin{document}
\maketitle

\begin{abstract}
Compression progress is a long-standing proposal for intrinsic motivation: reward an agent when its world model becomes better at predicting or compressing experience. The folk claim is that this reward is ``credible'' because it is paid only for learning. We make this precise and prove it. If intrinsic reward is the signed decrease of a fixed sealed-audit loss,
\[
 r_t^{\rm audit}=\E(\theta_{t-1})-\E(\theta_t),
\]
then cumulative reward telescopes exactly to endpoint audit improvement. Consequently no policy can drive reward upward indefinitely while true audit performance stagnates or degrades. For finite audit panels, the same result holds with a sharp false-positive budget: cumulative empirical reward is at most true audit improvement plus $2\Delta_n(\F,\delta)$, where $\Delta_n$ is the uniform audit deviation of the model class. This is horizon-free: adaptivity over time costs nothing once the sealed panel uniformly controls the class.

The theorem also identifies the failure modes. The guarantee disappears if progress is clipped, if progress is scored on the agent's own stream, if a reusable finite panel is exposed to a high-capacity model, or if a neural class makes $\Delta_n$ vacuous. We provide a Lean 4 mechanization of the structural core (telescoping, finite-audit Goodhart resistance conditional on uniform deviation, finite Gibbs nonnegativity, and the entropy-floor budget) and an experiment suite on ARC-TGI grid-transformation generators plus adaptive holdout attacks. The experiments confirm the theory: finite-audit deviation scales as $n^{-0.527}$; signed progress resists clip-farming, stream leakage, and noisy-TV curiosity; naive reusable audits are exploitable by black-box scalar feedback, while fresh subsampling, laddering, rounding, and one-shot release keep the attack below the $2\Delta_n$ threshold. These results delimit when compression progress is Goodhart-resistant: \emph{signed compression progress on a sealed audit is an accounting signal of genuine improvement.}
\end{abstract}

\section{Introduction}

Intrinsic motivation based on prediction or compression progress appears in Schmidhuber's work on artificial curiosity and the later compression-progress theory of interestingness \cite{schmidhuber1991curiosity,schmidhuber2008compression,schmidhuber2010formal}. Reward is paid only when the agent's model improves. This distinguishes learnable regularity from incompressible noise and should avoid noisy-TV pathologies that trap raw prediction-error bonuses.

But the informal statement is too broad. A learning agent can improve on its own recently selected stream while becoming worse on the target distribution. It can forget and relearn the same facts if reward clips away negative progress. It can overfit a finite validation set if repeated scalar feedback leaks information. It can exploit a high-capacity model class until a nominal holdout is no longer a holdout. These are the Goodhart channels that matter for intrinsic rewards in continual learning and recursive self-improvement.

We isolate the representation under which the compression-progress claim becomes true. Let $\Q$ be a fixed audit distribution and let $\E(\theta)=\mathbb E_{z\sim \Q}\ell(\theta,z)$ be audit log-loss or any lower-bounded proper scoring loss. Define signed audit compression progress by
\begin{equation}
  r_t^{\rm audit}=\E(\theta_{t-1})-\E(\theta_t).
\end{equation}
Then the entire reward history is an endpoint identity:
\begin{equation}
  \sum_{t=1}^T r_t^{\rm audit}=\E(\theta_0)-\E(\theta_T).
\end{equation}
Thus any apparent long-run reward must be paid for by a genuine reduction in audit loss. Goodhart resistance here is a property of the measurement frame: it holds because progress is scored against a fixed audit loss.

\paragraph{Contributions.}
We make four contributions.
First, we define \emph{budgeted Goodhart resistance}: a progress signal is credible up to a finite false-positive budget $\Gamma$ if cumulative reward cannot exceed true audit improvement by more than $\Gamma$. Exact sealed-audit compression progress has $\Gamma=0$; finite panels have $\Gamma=2\Delta_n$.
Second, we mechanize the structural core in Lean 4: exact telescoping, finite-audit Goodhart resistance under a uniform-deviation event, finite Gibbs nonnegativity, and an entropy-floor theorem for incompressible components.
Third, we separate the reward signal from the scheduler: audit compression progress supplies the credible reward, while multiplicative weights / EXP3 provides allocation.
Fourth, we run a focused experiment suite using ARC-TGI task generators, RND \cite{burda2018rnd}, ICM \cite{pathak2017icm}, prediction-error curiosity, finite-audit concentration checks, stream leakage, clipping cycles, reusable-panel memorization, and black-box scalar-feedback holdout attacks.

\begin{figure}[H]
  \centering
  \includegraphics[width=0.96\linewidth]{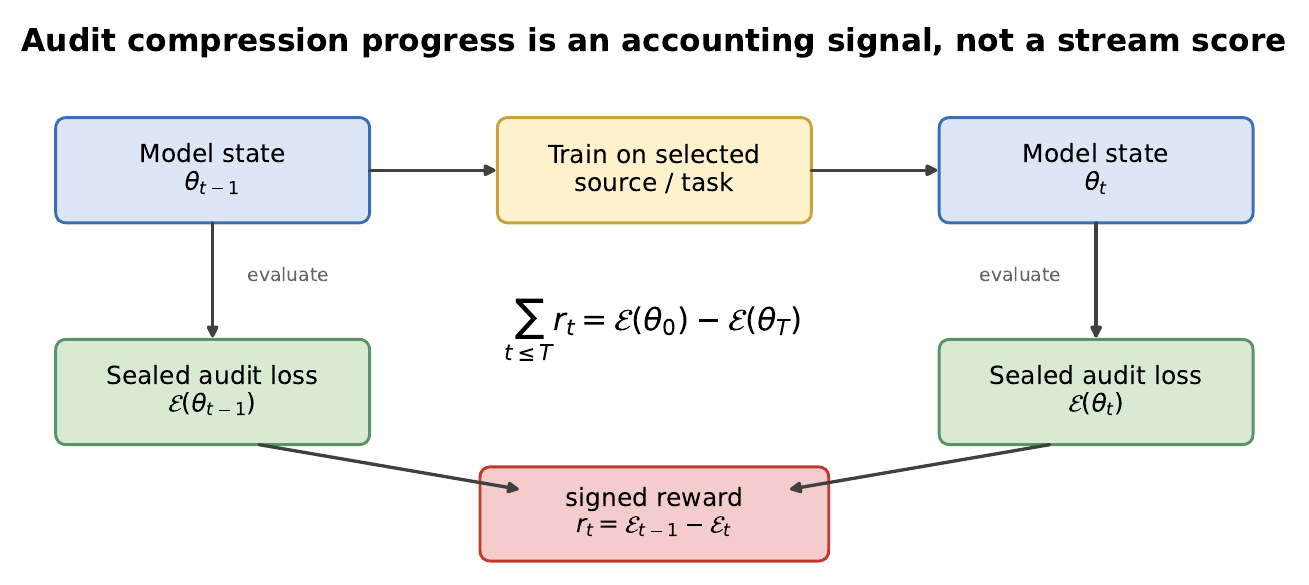}
  \caption{The measurement frame. Training data may be selected adaptively, but reward is computed only from the signed change in a fixed audit loss. This makes intrinsic reward an endpoint accounting identity over the sealed audit.}
  \label{fig:accounting}
\end{figure}

\section{Related work}

\paragraph{Goodhart's law and reward hacking.} Optimizing a proxy until it diverges from the target it stands for is Goodhart's law \cite{manheim2018goodhart}, and its learned-agent form is reward hacking or specification gaming \cite{amodei2016concrete,lehman2020creativity,skalse2022reward,pan2022reward}. Most of this literature characterizes when a fixed proxy is unsafe to optimize. We hold the measurement frame fixed and ask a quantitative question: when can a progress signal exceed true improvement, and by how much? Budgeted Goodhart resistance answers this for compression-progress rewards with a finite false-positive budget.

\paragraph{Intrinsic motivation and compression progress.} Compression progress as a driver of curiosity and creativity originates with Schmidhuber \cite{schmidhuber1991curiosity,schmidhuber2008compression,schmidhuber2010formal}. Prediction-error and feature-prediction bonuses \cite{pathak2017icm} and random network distillation \cite{burda2018rnd} are the standard deep reinforcement learning realizations. These bonuses score error or novelty on the agent's own stream. Audit compression progress scores signed error reduction on a sealed distribution, which is what produces the endpoint accounting identity and the entropy floor.

\paragraph{Adaptive data analysis and holdout reuse.} Repeated queries against a finite validation set erode its guarantees; the reusable holdout and the Ladder mechanism bound this erosion \cite{dwork2015holdout,blum2015ladder}, and empirical studies measure it on real leaderboards and benchmarks \cite{recht2019imagenet,roelofs2019metaanalysis}. The finite-panel budget $2\Delta_n$ is the audit-CP form of the same phenomenon, and our adaptive scalar-feedback attack instantiates it together with the standard release defenses.

\paragraph{Proper scoring rules.} Log-loss is a strictly proper scoring rule \cite{gneiting2007scoring}, so its population minimizer is the true conditional distribution. The entropy floor (Theorem~\ref{thm:entropy-floor}) is the statement that this minimum equals the conditional entropy, which is why a purely random component carries only a finite compression-progress budget. The calibration probe separates this proper-scoring signal from hard accuracy.

\section{Setup: audit compression progress}

Let $\Theta$ be a class of model states and let $\ell:\Theta\times\mathcal Z\to\R$ be a bounded or lower-bounded predictive loss. In the log-loss case, $\ell(\theta,(x,y))=-\log p_\theta(y\mid x)$. For experiments we use probability-floored cross-entropy,
\begin{equation}
  \ell_\varepsilon(\theta,(x,y))=-\log \max\{\varepsilon,p_\theta(y\mid x)\},
\end{equation}
which is bounded by $R=-\log\varepsilon$.

\begin{definition}[Sealed audit loss]
A sealed audit distribution $\Q$ is fixed independently of the agent's adaptive training trajectory and cannot be selected, distorted, or inspected by the agent except through the permitted audit-release mechanism. The population audit loss is
\begin{equation}
  \E(\theta)=\mathbb E_{z\sim \Q}\ell(\theta,z).
\end{equation}
For a finite audit panel $A_n=(z_1,\ldots,z_n)$, the empirical audit loss is
\begin{equation}
  \hE_n(\theta)=\frac1n\sum_{i=1}^n \ell(\theta,z_i).
\end{equation}
\end{definition}

\begin{definition}[Signed audit compression progress]
Given a trajectory $\theta_0,\theta_1,\ldots$, signed audit compression progress is
\begin{equation}
  r_t^{\CP}=\E(\theta_{t-1})-\E(\theta_t),\qquad
  \hat r_t^{\CP}=\hE_n(\theta_{t-1})-\hE_n(\theta_t).
\end{equation}
The sign is part of the definition: negative progress is charged back to the agent.
\end{definition}

\begin{definition}[False-positive budget]
For a reward signal $r_t$ and true audit loss $\E$, define the reward excess at horizon $T$ by
\begin{equation}
  \Gamma_T(r,\E)=\sum_{t=1}^T r_t-\big(\E(\theta_0)-\E(\theta_T)\big).
\end{equation}
A signal is $\Gamma$-Goodhart-resistant on a class of trajectories if $\Gamma_T(r,\E)\le \Gamma$ for every horizon $T$ and every admissible trajectory in the class. Exact audit-CP has $\Gamma=0$; finite-panel audit-CP has $\Gamma=2\Delta_n$ on the uniform-deviation event.
\end{definition}

This condition is stronger than correlation with learning: it caps the apparent reward obtainable without true audit improvement.

\section{Theorems}

\subsection{Exact sealed audits: zero false-positive budget}

\begin{theorem}[Exact-audit telescoping and finite budget]
\label{thm:telescope}
Let $\E:\Theta\to\R$ and let $\theta_t$ be any trajectory. Define $r_t=\E(\theta_{t-1})-\E(\theta_t)$. Then for every horizon $T$,
\begin{equation}
  \sum_{t=1}^T r_t=\E(\theta_0)-\E(\theta_T).
\end{equation}
If $\E(\theta_T)\ge E_{\min}$, then
\begin{equation}
  \sum_{t=1}^T r_t\le \E(\theta_0)-E_{\min}.
\end{equation}
Thus no policy can make cumulative signed audit progress diverge while audit loss stagnates or remains lower-bounded.
\end{theorem}

\begin{proof}
The sum telescopes:
\[
\sum_{t=1}^T \big(\E(\theta_{t-1})-\E(\theta_t)\big)=\E(\theta_0)-\E(\theta_T).
\]
The lower-bound statement follows immediately. This proof is mechanized as \lean{cumCP\_telescope}; the finite-budget form is \lean{cumCP\_le\_of\_lb}.
\end{proof}

Each hypothesis of Theorem~\ref{thm:telescope} is necessary: a fixed audit loss provides a single potential to telescope, signed accounting lets negative terms cancel, and a lower bound yields a finite budget.

\subsection{Finite audits: the \texorpdfstring{$2\Delta_n$}{2 Delta n} false-positive budget}

A reusable finite audit panel is not automatically sealed in the population sense; the relevant condition is a uniform-deviation event.

\begin{definition}[Uniform audit deviation]
For a class $\F\subseteq \Theta$, define
\begin{equation}
  \Delta_n(\F)=\sup_{\theta\in\F}|\hE_n(\theta)-\E(\theta)|.
\end{equation}
We say the panel realizes deviation $\Delta$ on $\F$ if $\Delta_n(\F)\le \Delta$.
\end{definition}

\begin{theorem}[Finite-audit Goodhart resistance]
\label{thm:finite-audit}
Assume $\Delta_n(\F)\le \Delta$ and $\theta_t\in\F$ for all $t\le T$. Then
\begin{equation}
  \sum_{t=1}^T \hat r_t^{\CP}
  \le
  \sum_{t=1}^T r_t^{\CP}+2\Delta
  = \E(\theta_0)-\E(\theta_T)+2\Delta.
\end{equation}
Equivalently, empirical audit-CP has false-positive budget at most $2\Delta$.
\end{theorem}

\begin{proof}
By telescoping,
\[
\sum_t \hat r_t=\hE_n(\theta_0)-\hE_n(\theta_T),\qquad
\sum_t r_t=\E(\theta_0)-\E(\theta_T).
\]
Uniform deviation controls only the two endpoints:
\[
\hE_n(\theta_0)\le \E(\theta_0)+\Delta,
\qquad
\hE_n(\theta_T)\ge \E(\theta_T)-\Delta.
\]
Combining gives the result. This is mechanized as \lean{finite\_audit\_goodhart}.
\end{proof}

There is no union bound over $T$: after signed telescoping, the adaptive history reduces to endpoint control. The cost of adaptivity appears in proving that the panel realizes a uniform-deviation event for the reachable class; the theorem itself is horizon-free.

\begin{corollary}[Finite experts]
\label{cor:finite-experts}
If $|\F|=N$ and $\ell\in[0,R]$, then with probability at least $1-\delta$ over an i.i.d. audit panel,
\begin{equation}
  \Delta_n(\F)\le R\sqrt{\frac{\log(2N/\delta)}{2n}},
\end{equation}
so
\begin{equation}
  \sum_{t=1}^T \hat r_t^{\CP}
  \le \E(\theta_0)-\E(\theta_T)+2R\sqrt{\frac{\log(2N/\delta)}{2n}}.
\end{equation}
\end{corollary}

\begin{proof}
Apply two-sided Hoeffding to each fixed model and union bound over $\F$; then invoke Theorem~\ref{thm:finite-audit}. The Lean artifact mechanizes the deterministic implication from a realized uniform-deviation event to Goodhart resistance; this probabilistic instantiation is the standard finite-class concentration corollary.
\end{proof}

For infinite classes, replace $N$ by the appropriate covering number, Rademacher complexity, or PAC-Bayesian radius. In particular, bounded linear balls and bounded RKHS balls produce the same form: a finite false-positive budget whenever the effective audit capacity is finite. General neural networks enter the theory only through their effective class size or stability. If the reachable class is large enough to memorize the reusable audit panel, the bound becomes vacuous; this delimits where finite-audit resistance ceases to hold.

\subsection{Entropy floor: why noisy TV cannot pay forever}

Prediction-error bonuses pay for error itself and are therefore attracted to irreducible noise. Compression progress pays for error reduction. For log-loss, the distinction is formalized by the following entropy floor.

\begin{theorem}[Entropy floor]
\label{thm:entropy-floor}
Let $S$ be an audit component with conditional distribution $\Q_S(Y\mid X)$ and log-loss risk $\E_S(\theta)$. If $\E_S(\theta)\ge H_{\Q_S}(Y\mid X)$ for all $\theta$, then signed compression progress on $S$ satisfies
\begin{equation}
  \sum_{t=1}^T\big(\E_S(\theta_{t-1})-\E_S(\theta_t)\big)
  \le \E_S(\theta_0)-H_{\Q_S}(Y\mid X).
\end{equation}
A purely random component has only a finite improvement budget; once the model reaches the entropy floor, it cannot keep paying audit-CP.
\end{theorem}

\begin{proof}
For log-loss, $\E_S(\theta)=H_{\Q_S}(Y\mid X)+\mathrm{KL}(\Q_S\Vert p_\theta)$ and $\mathrm{KL}\ge 0$, hence $\E_S(\theta)\ge H_{\Q_S}(Y\mid X)$, and the budget follows from Theorem~\ref{thm:telescope}. This composition is mechanized end to end in the conditional, input-averaged form: the cross-entropy decomposition and the finite Gibbs inequality discharge the conditional entropy floor, which is then fed into the telescoping budget. The corresponding Lean declarations are listed in Table~\ref{tab:lean}; the single-input case specializes the same chain, and an abstract version takes the floor as a hypothesis for an arbitrary lower-bounded potential.
\end{proof}

\section{Where the theorem breaks}

Each assumption above has a failure construction and a corresponding experiment.

\paragraph{Clipping destroys accounting.}
If reward is $r_t^+=\max\{0,\E(\theta_{t-1})-\E(\theta_t)\}$, then a two-state cycle $a,b,a,b,\ldots$ with $\E(a)<\E(b)$ accumulates positive reward on every $b\to a$ transition while returning to the same endpoint every two steps. Thus $\sum_t r_t^+$ can grow linearly at zero net improvement. Signed progress cancels exactly.

\paragraph{Stream scoring destroys the fixed potential.}
If the agent is rewarded by $\E_{\Pstream_t}(\theta_{t-1})-\E_{\Pstream_t}(\theta_t)$ on its own selected stream $\Pstream_t$, then there is no fixed $\E$ to telescope. A policy can select or distort streams where local loss decreases while sealed-audit loss does not.

\paragraph{High-capacity reusable panels destroy uniform deviation.}
If $\F$ can interpolate the finite audit panel, a policy can drive $\hE_n$ down while $\E$ stagnates or rises. This is ordinary adaptive holdout overfitting \cite{dwork2015holdout,blum2015ladder}, now expressed as false audit-CP. The $2\Delta_n$ theorem remains correct, but $\Delta_n$ is no longer small.

\begin{figure}[H]
  \centering
  \includegraphics[width=0.98\linewidth]{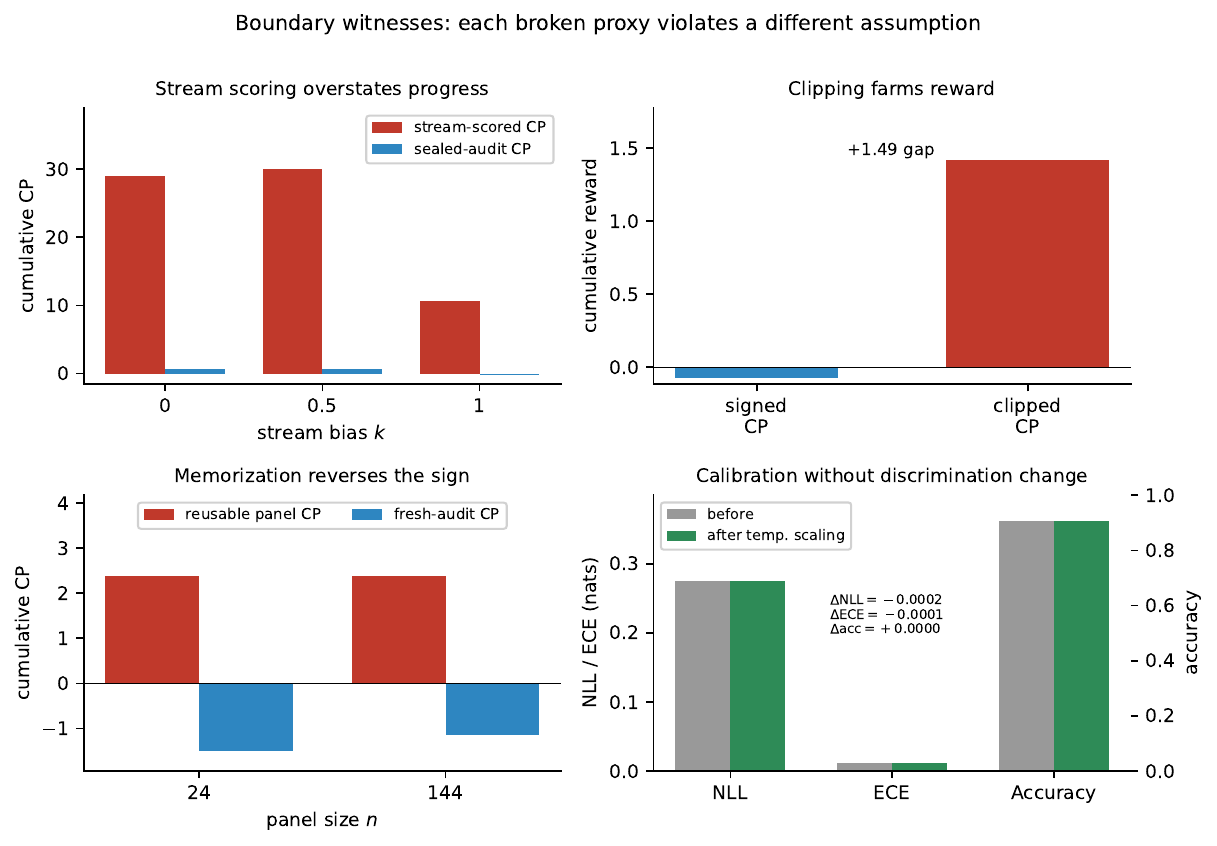}
  \caption{Boundary tests from the experiment artifact. Stream scoring: stream-scored progress can exceed sealed-audit progress by roughly $40\times$. Clipping: clipped reward farms a positive cumulative signal while signed reward equals endpoint change. Panel memorization: direct reusable-panel training yields positive apparent progress and negative true progress. Calibration: temperature scaling changes log-loss/ECE at fixed hard accuracy, emphasizing that compression progress measures probabilistic prediction.}
  \label{fig:boundary}
\end{figure}

\section{Algorithmic separation: reward signal versus scheduler}

We distinguish the reward signal from the scheduler.

\paragraph{Audit-CP is the reward/accounting signal.}
For each candidate training source $j$, temporarily train or estimate the update effect, evaluate the sealed-audit loss before and after, and release the signed reward
\begin{equation}
  r_{t,j}=\hE_n(\theta_t)-\hE_n(U_j(\theta_t)),
\end{equation}
where $U_j$ denotes the update produced by source $j$. This reward is meaningful because Theorem~\ref{thm:finite-audit} controls its cumulative false positives.

\paragraph{EXP3/MWU is the scheduler.}
We use multiplicative weights / EXP3 to allocate training among sources. With weights $w_{j,t}$ and exploration parameter $\gamma$,
\begin{equation}
 p_{j,t}=(1-\gamma)\frac{w_{j,t}}{\sum_k w_{k,t}}+\frac{\gamma}{K},\qquad
 w_{j,t+1}=w_{j,t}\exp\left(\eta\frac{r_{t,j}\mathbf 1\{j=j_t\}}{p_{j,t}}\right).
\end{equation}
This is the standard adversarial-bandit use of multiplicative weights \cite{auer2002nonstochastic}, providing exploration and exploitation. The credibility of the cumulative reward comes from the reward signal: if the payoff fed to MWU is a hackable stream reward, MWU efficiently optimizes the hack; if the payoff is signed audit-CP, Theorem~\ref{thm:finite-audit} supplies the credibility certificate.

\begin{center}
\fbox{\begin{minipage}{0.92\linewidth}
\textbf{Audit-CP scheduler protocol.}
For $t=1,\ldots,T$: choose a source $j_t$ with an EXP3/MWU policy; train the model on fresh samples from source $j_t$; compute signed audit loss change on the sealed panel; update the scheduler using the signed audit-CP reward. Negative rewards are retained. The audit panel is never used as training data and the released signal is protected by the chosen audit-release mechanism.
\end{minipage}}
\end{center}

\section{Experiments}

\subsection{Experimental substrate}

All reported numbers come from the accompanying experiment artifact. The curriculum substrate is ARC-TGI, a generator framework for ARC-style grid transformation tasks; its released generator inventory includes ARC-Mini, ARC-AGI-1, and ARC-AGI-2 families \cite{lehmann2026arctgi}. Curriculum experiments require repeated fresh samples from a stable latent rule; a single static puzzle instance is insufficient. The broader ARC-AGI-2 benchmark is motivated by few-shot abstraction and refinement-loop behavior \cite{chollet2026arcprize}.

The experiment suite uses $30\times 30$ padded grids, 24 learnable ARC-TGI grid-to-grid families, 8 i.i.d. uniform distractor families, probability-floored cross-entropy with $\varepsilon=0.01$ and cap $R\approx 4.6$, audit panels of size 512 per family unless otherwise stated, 5000 training steps, and 20 seeds for every experiment except the adaptive attack. The adaptive scalar-feedback attack reports its single-cell power calibration and panel-size scaling curve over 20 seeds each; the full-scale grid attack, in which panel capacity far exceeds the query budget, uses 12 seeds. Baselines include prediction-error curiosity, real RND with a fixed random target network and learned predictor \cite{burda2018rnd}, real ICM with inverse-model features and forward prediction error \cite{pathak2017icm}, uniform sampling, round-robin, and an oracle that samples only learnable families.

\subsection{Main result table}

\begin{table}[H]
\centering
\small
\begin{tabular}{p{0.20\linewidth}p{0.38\linewidth}p{0.34\linewidth}}
\toprule
Experiment & Question & Result \\
\midrule
Finite-audit concentration & Does finite-audit noise shrink at the theory rate? & $\Delta_n\propto n^{-0.527}$ over 20 seeds, matching the $n^{-1/2}$ finite-class rate. \\
Reward-signal ablation & Does audit-CP resist noisy-TV curriculum Goodharting? & Audit-CP is the best non-oracle reward signal: active-cell accuracy $0.387\pm0.006$ vs prediction-error $0.338\pm0.026$, RND $0.371\pm0.004$, ICM $0.347\pm0.009$, uniform $0.376\pm0.005$, oracle $0.391\pm0.005$. \\
Holdout-reuse attack & Can scalar feedback overfit a reusable finite audit? & Yes for naive release: gap $3.24\pm0.07>2\Delta_n=1.64$, with the attacker winning in $20/20$ seeds. Fresh subsampling, laddering, rounding, and one-shot release stay below threshold. \\
Signed vs.\ clipped & Is signed accounting necessary? & Clipped CP exceeds signed endpoint progress by $1.490\pm0.072$ in 20/20 seeds. \\
Stream vs.\ audit & Is stream CP a valid proxy for audit CP? & No. Stream CP overstates sealed-audit CP by approximately $40\times$ at $k=0$ and $k=0.5$. \\
Panel memorization & Can reusable panels be memorized directly? & Yes. Apparent CP $+2.39$ while fresh true CP is $-1.50$ at $n=24$; false-positive excess $3.88\pm0.20$. \\
Calibration probe & Does log-loss CP reduce to hard accuracy? & No. Temperature rescaling changes NLL/ECE at the noise floor while $\Delta$accuracy is exactly $0.00000$ in all 20 seeds. \\
\bottomrule
\end{tabular}
\caption{Experiment suite. Each row tests one theorem assumption: concentration, noisy-TV entropy, adaptive holdout release, signed accounting, stream/audit separation, capacity, and metric identity. Reported uncertainties are one standard deviation across the seeds (20 unless otherwise noted).}
\label{tab:main-results}
\end{table}

\subsection{Finite-audit concentration}

The finite-audit theorem is useful only if $\Delta_n$ is small for the reachable class. We measure empirical uniform deviation over finite model families and vary panel size.

\begin{figure}[H]
  \centering
  \includegraphics[width=0.58\linewidth]{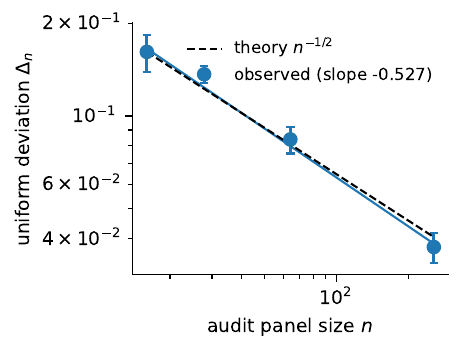}
  \caption{Finite-audit concentration: empirical uniform deviation of floored cross-entropy over a finite family decays as $\Delta_n\propto n^{-0.527}$, close to the $n^{-1/2}$ prediction.}
  \label{fig:e5}
\end{figure}

The observed slope is $-0.527$ over 20 seeds, matching the predicted $n^{-1/2}$ rate; the Hoeffding constant itself is loose. The false-positive budget is therefore an empirically measurable quantity.

\subsection{Reward-signal ablation under noise distractors}

The ablation holds the scheduler fixed and varies only the reward signal. The win condition is downstream active-cell reconstruction accuracy on a sealed audit panel, plus low allocation to i.i.d. noise distractors.

\begin{figure}[H]
  \centering
  \includegraphics[width=0.96\linewidth]{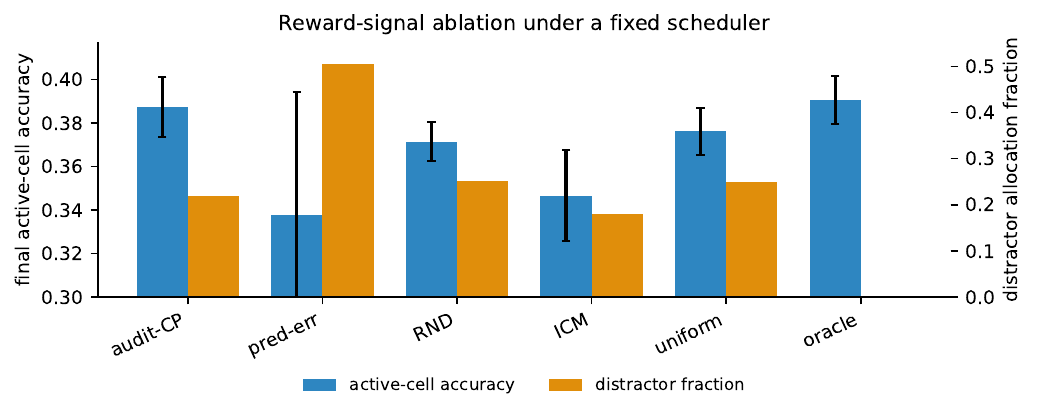}
  \caption{Reward-signal ablation with the same scheduler. Final active-cell reconstruction accuracy (left axis) and distractor allocation fraction (right axis) per reward signal, 20 seeds. Audit-CP is the strongest non-oracle reward signal and spends less than half the distractor budget of prediction-error curiosity.}
  \label{fig:c1}
\end{figure}

\begin{table}[H]
\centering
\small
\begin{tabular}{lccc}
\toprule
Reward signal & Active-cell accuracy & Floored CE & Distractor fraction \\
\midrule
Audit-CP & $\mathbf{0.387\pm0.006}$ & $1.748\pm0.013$ & $0.219\pm0.009$ \\
Prediction error & $0.338\pm0.026$ & $1.845\pm0.053$ & $0.504\pm0.021$ \\
RND & $0.371\pm0.004$ & $1.751\pm0.012$ & $0.251\pm0.012$ \\
ICM & $0.347\pm0.009$ & $1.808\pm0.027$ & $0.179\pm0.024$ \\
Uniform & $0.376\pm0.005$ & $1.773\pm0.013$ & $0.250\pm0.003$ \\
Round-robin & $0.371\pm0.006$ & $1.778\pm0.012$ & $0.250\pm0.000$ \\
Oracle learnable & $0.391\pm0.005$ & $\mathbf{1.728\pm0.012}$ & $0.000\pm0.000$ \\
\bottomrule
\end{tabular}
\caption{Reward-signal ablation, final metrics, 20 seeds. Audit-CP beats prediction-error by $+0.0498$ and ICM by $+0.0407$ in 20/20 seeds, edges RND by $+0.0159$ in 17/20 seeds, and approaches the learnable-only oracle.}
\label{tab:c1}
\end{table}

Prediction error is Goodharted by noise: it spends 50.4\% of its sampling budget on distractors. Audit-CP spends 21.9\%, remains above uniform and round-robin, and approaches the oracle. Audit-CP allocates toward tasks whose error compresses, independent of task difficulty.

\begin{figure}[H]
  \centering
  \includegraphics[width=0.9\linewidth]{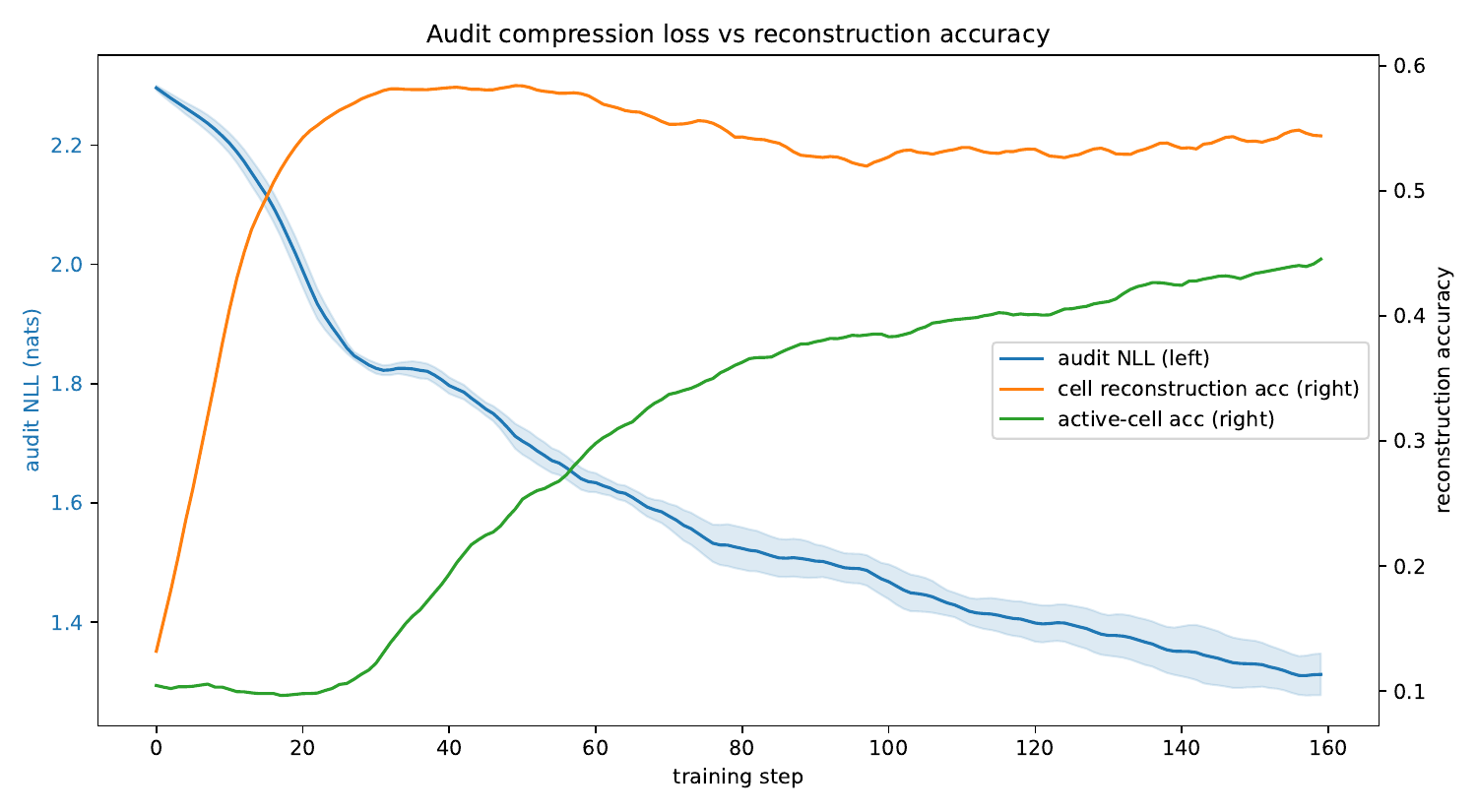}
  \caption{Sealed-panel learning curves for an audit-CP run. Audit NLL (left axis) falls, equivalently cumulative audit compression progress rises, while cell and active-cell reconstruction accuracy (right axis) improve over training. The signed audit-CP signal tracks downstream reconstruction. One-step log-loss changes differ from one-step hard-accuracy changes; the reward-signal ablation reports both metrics.}
  \label{fig:exp1}
\end{figure}

\subsection{Adaptive scalar-feedback holdout overfitting}

In the black-box scalar-feedback setting, the attacker sees only released audit scores and proposes adaptive model updates. This is the holdout-reuse problem applied to audit-CP.

\begin{figure}[H]
  \centering
  \includegraphics[width=0.96\linewidth]{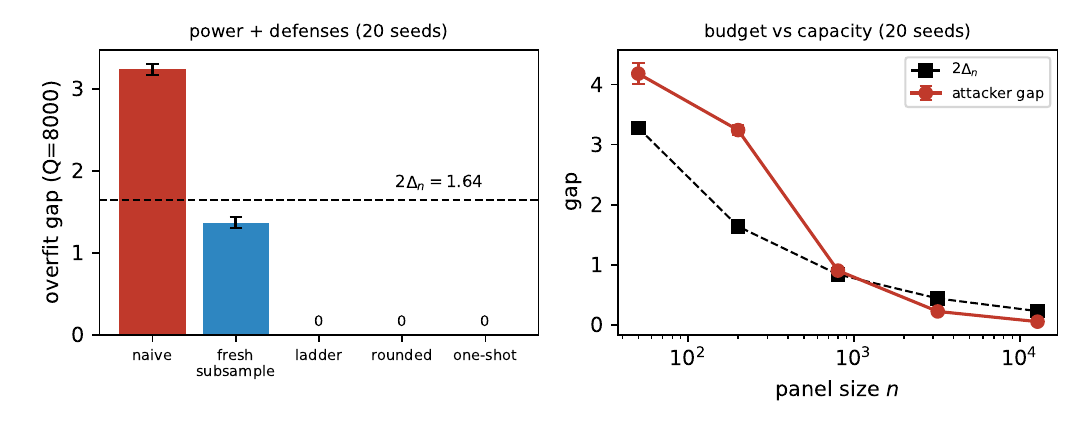}
  \caption{Adaptive scalar-feedback overfitting (single-cell calibration, 20 seeds; error bars are one standard deviation). Left: the attacker has power against naive reusable release, with gap $3.24\pm0.07>2\Delta_n=1.64$ and a win in all 20 seeds. Fresh subsampling, laddering, rounding, and one-shot release keep the gap below threshold. Right: the gap falls below threshold once panel capacity exceeds the query budget, crossing in $20/20$ seeds at $n\le 800$ and in $0/20$ at $n\ge 3200$.}
  \label{fig:e7}
\end{figure}

\begin{table}[H]
\centering
\small
\begin{tabular}{lccc}
\toprule
Release mechanism & Overfit gap (mean $\pm$ 1 SD) & $2\Delta_n$ & Attacker wins \\
\midrule
Naive reusable & $3.24\pm0.07$ & $1.64$ & $20/20$ \\
Fresh subsample & $1.37\pm0.07$ & $1.64$ & $0/20$ \\
Ladder & $0.00$ & $1.64$ & $0/20$ \\
Rounded & $0.00$ & $1.64$ & $0/20$ \\
One-shot & $0.00$ & $1.64$ & $0/20$ \\
\bottomrule
\end{tabular}
\caption{Holdout-attack defense calibration at $Q=8000$, $n=200$, single-cell attack, 20 seeds. The overfit gap is reusable minus fresh audit-CP; the attacker wins when the gap exceeds the finite-audit threshold $2\Delta_n$. Naive reuse is exploitable in all 20 seeds; the standard release defenses suppress the scalar-feedback attack below threshold.}
\label{tab:e7}
\end{table}

This attack maps the boundary of finite audit-CP: reusing a panel and releasing too much information lets an adaptive agent overfit it, whereas defenses that reduce information leakage keep the scalar-feedback gap below the validity threshold. The $2\Delta_n$ threshold is therefore operational.

\subsection{Boundary and metric tests}

Figure~\ref{fig:boundary} collects the remaining assumption tests. Signed accounting is load-bearing: signed cumulative progress equals endpoint improvement $-0.072\pm0.042$, while clipped reward accumulates $1.417\pm0.088$. The stream/audit separation holds: stream-scored CP is $29.0$ to $30.0$ while sealed-audit CP is $0.72$ to $0.74$ for $k\in\{0,0.5\}$, and at $k=1$ the sealed audit goes slightly negative. The capacity boundary is real: direct training on a reusable panel gives apparent CP around $+2.39$ while true fresh-audit CP is negative. Calibration and discrimination decouple: log-loss and ECE can change with no change in hard accuracy, so compression progress should be read as probabilistic predictive improvement, not merely 0-1 discrimination.

\section{Lean 4 mechanization}

The formal artifact is a self-contained Lean 4 / Mathlib development. The core declarations are listed in Table~\ref{tab:lean}. The two structural theorems are stated for an arbitrary potential $\E:\iota\to\R$; cross-entropy and finite PMFs instantiate the entropy-floor theorem.

\begin{figure}[H]
  \centering
  \includegraphics[width=0.92\linewidth]{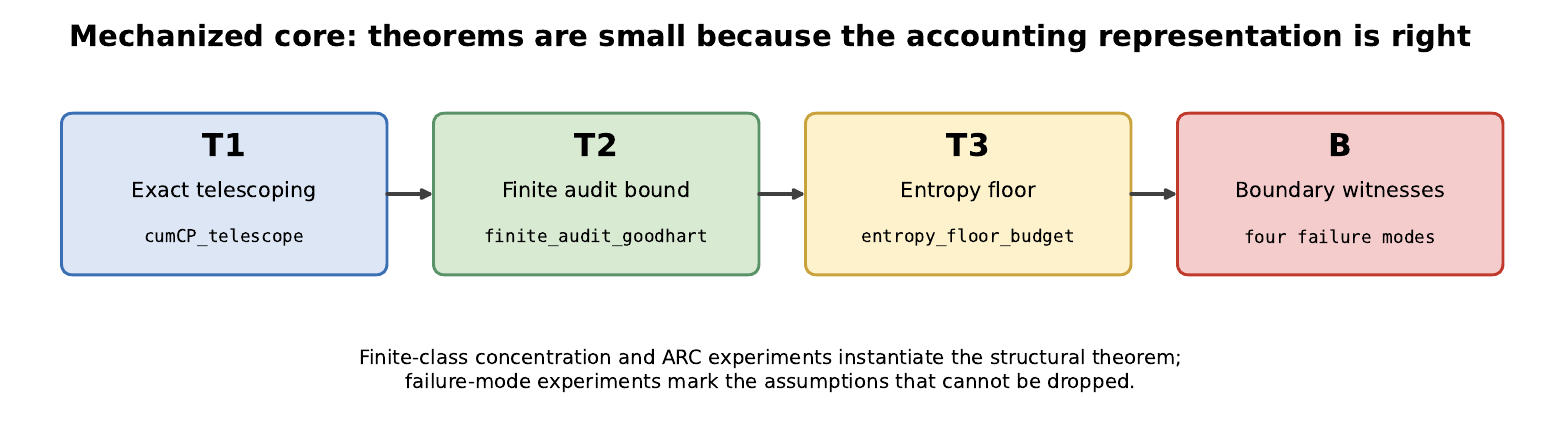}
  \caption{Proof architecture. Signed audit progress is an endpoint potential drop, which keeps the mechanized theorems short. The finite-class and ARC experiments instantiate the structural result; the boundary experiments mark assumptions that cannot be dropped.}
  \label{fig:lean-ladder}
\end{figure}

\begin{table}[H]
\centering
\small
\begin{tabular}{p{0.22\linewidth}p{0.32\linewidth}p{0.38\linewidth}}
\toprule
Paper result & Lean declaration & Meaning \\
\midrule
T1 exact telescoping & \lean{cumCP\_telescope} & $\sum_t(E(g_t)-E(g_{t+1}))=E(g_0)-E(g_T)$ \\
T1 finite budget & \lean{cumCP\_le\_of\_lb} & If the terminal loss is lower-bounded, cumulative signed CP is bounded. \\
T2 finite-audit bound & \lean{finite\_audit\_goodhart} & On a uniform-deviation event, empirical audit-CP exceeds true audit improvement by at most $2\Delta$. \\
Finite Gibbs & \lean{klDivFinitePMF\_nonneg} & KL between finite PMFs is nonnegative under positive prior support. \\
Cross-entropy floor & \lean{entropy\_le\_crossEntropyFinitePMF} & $H(Q)\le H(Q,P)$, derived from Gibbs via \lean{crossEntropyFinitePMF\_eq\_entropy\_add\_kl}. \\
T3 entropy-floor budget & \lean{cond\_entropy\_floor\_budget\_crossEntropy} & Conditional (input-averaged) cross-entropy CP budget $H(Y\mid X)\le\E_S$, with the floor discharged by conditional Gibbs (\lean{condEntropy\_le\_condCrossEntropy}); \lean{entropy\_floor\_budget\_crossEntropy} is the single-input case. \\
Finite-expert radius & \lean{deltaFiniteExperts} & Definition of $L\sqrt{\log(2N/\delta)/(2n)}$ used by the finite-class corollary. \\
\bottomrule
\end{tabular}
\caption{Lean proof integration. The deterministic and information-theoretic core is mechanized. The finite-class concentration corollary uses the standard Hoeffding-plus-union-bound argument to realize the uniform-deviation premise of \lean{finite\_audit\_goodhart}.}
\label{tab:lean}
\end{table}

The finite-audit bound is mechanized as follows:

\begin{lstlisting}[style=leanstyle]
theorem finite_audit_goodhart {iota : Type*} {F : Set iota}
    {Ehat E : iota -> Real} {Delta : Real}
    (hUD : UniformDev F Ehat E Delta) {g : Nat -> iota}
    (hg : Admissible g F) (T : Nat) :
    cumCP Ehat g T <= cumCP E g T + 2 * Delta := by
  rw [cumCP_telescope Ehat g T, cumCP_telescope E g T]
  have h0 := abs_le.1 (hUD (g 0) (hg 0))
  have hT := abs_le.1 (hUD (g T) (hg T))
  linarith [h0.1, h0.2, hT.1, hT.2]
\end{lstlisting}

Once signed empirical CP telescopes, only the endpoints matter: the uniform-deviation event is used exactly twice, at $g(0)$ and $g(T)$. The bound is therefore horizon-free.

\section{Discussion}

\paragraph{Relation to compression-progress folklore.}
The informal compression-progress principle rewards improvements in the compressor. Our theorem states the conditions under which this is credible: the compressor must be evaluated by a fixed audit loss, the reward must be signed, and finite reuse must be controlled by a uniform-deviation budget.

\paragraph{Multiplicative weights versus the audit certificate.}
Multiplicative weights has regret and allocation guarantees, but its potential is not a predictive-performance metric: it can reweight tasks without reducing audit loss. The credibility certificate is therefore audit-CP, and MWU/EXP3 is a scheduler that consumes it.

\paragraph{Log-loss versus hard accuracy.}
Compression is probabilistic. Proper scoring losses reward calibrated beliefs, not only argmax decisions. Calibration can improve at fixed hard accuracy. For reconstruction tasks we report active-cell accuracy for interpretability; the certified signal is audit cross-entropy.

\section{Limitations}

The guarantee is conditional: it assumes a sealed audit, a reachable model class of finite effective capacity, and an audit the agent cannot alter, the boundaries probed by the stream-scoring, panel-memorization, and holdout-reuse experiments. The deterministic and information-theoretic core is mechanized in Lean; the finite-class concentration corollary is the standard Hoeffding-plus-union argument, stated but not mechanized. The empirical claims are scoped to match: audit-CP leads the non-oracle baselines only narrowly in mean active-cell accuracy, so its robust separation is the distractor-allocation budget; the unchanged accuracy in the calibration probe is an identity of temperature rescaling, with the calibration signal carried by NLL and ECE; and the scalar-feedback attack crosses threshold in the single-cell calibration while the deployed-scale panel stays below it.

\section{Conclusion}

Compression progress becomes a Goodhart-resistant RSI signal only after a precise measurement choice: score \emph{signed} progress on a \emph{sealed audit} distribution. The resulting reward has an endpoint accounting identity and, for finite panels, a $2\Delta_n$ false-positive budget. The same budget accounts for the failure modes: prediction-error curiosity is attracted to irreducible noise; clipped progress farms forgetting/relearning cycles; stream progress optimizes the wrong distribution; high-capacity reusable audits can be overfit. Audit-CP avoids these failures exactly when the audit remains sealed and its uniform-deviation budget is finite.

The resulting design rule for intrinsic-motivation and self-improvement systems is: use compression progress as the audited reward, use a separate scheduler for exploration, retain negative progress, and treat the audit as a protected measurement instrument. Under these conditions, apparent improvement is bounded by true audit improvement plus the panel's deviation budget.

\section*{Ethics statement}

This work studies an accounting signal for genuine capability improvement, motivated by reward design for continual learning and recursive self-improvement. The contribution is a measurement condition, not a safety guarantee. The telescoping identity and the $2\Delta_n$ budget certify that signed compression progress on a sealed audit cannot overstate true audit improvement by more than a finite amount. They do not certify that the audit distribution captures everything a deployer cares about, nor that an agent which satisfies the budget is safe. We state the boundary conditions explicitly, namely clipping, stream scoring, high-capacity reusable panels, and adaptive feedback, so that the signal is not read as a stronger guarantee than it is.

The adaptive scalar-feedback attack reproduces a known adaptive-data-analysis vulnerability: repeated scalar feedback from a reusable panel can be overfit. We include it to calibrate the test and to show that standard release defenses suppress it, not to introduce a new exploit. All experiments use synthetic ARC-TGI grid-transformation tasks and contain no human-subject or personal data. The principal societal risk is overclaiming: deploying audit-CP without a sealed audit and a finite uniform-deviation budget could create false assurance of genuine improvement. The boundary results are included to make that failure mode visible.

\section*{Reproducibility statement}

The numerical claims in this paper are computed from the accompanying experiment artifact containing scripts, logs, per-seed JSON shards, summary tables, and vector figures. The formal claims in Table~\ref{tab:lean} are represented by the accompanying Lean 4 artifact with a pinned Mathlib dependency and an axiom-audit target. The public source package for this paper uses only these two artifacts as inputs for results and proof declarations.

\clearpage
\appendix

\section{Additional theorem details}

\subsection{Finite-class concentration}

Let $X_i^\theta=\ell(\theta,z_i)$ for $z_i\sim \Q$, with $X_i^\theta\in[0,R]$. For a fixed $\theta$, Hoeffding gives
\begin{equation}
  \Pr\left(|\hE_n(\theta)-\E(\theta)|>\epsilon\right)
  \le 2\exp\left(-\frac{2n\epsilon^2}{R^2}\right).
\end{equation}
Union bounding over $N$ experts gives
\begin{equation}
  \Pr\left(\sup_{\theta\in\F}|\hE_n(\theta)-\E(\theta)|>\epsilon\right)
  \le 2N\exp\left(-\frac{2n\epsilon^2}{R^2}\right).
\end{equation}
Setting the right-hand side to $\delta$ yields Corollary~\ref{cor:finite-experts}. Combining with Theorem~\ref{thm:finite-audit} gives the finite-expert false-positive budget.

\subsection{Covering-number extension}

For a bounded Lipschitz loss family $\ell_\theta$ and an $\epsilon$-cover of size $N(\epsilon,\F,d)$ in a metric that controls loss uniformly, the same argument yields
\begin{equation}
  \Delta_n(\F)\lesssim \epsilon + R\sqrt{\frac{\log N(\epsilon,\F,d)+\log(1/\delta)}{n}}.
\end{equation}
For bounded linear classes in $d$ dimensions, $\log N(\epsilon)$ scales as $d\log(B/\epsilon)$ under standard norm constraints. For bounded RKHS balls, the analogous statement is controlled by the corresponding metric entropy or effective dimension. The paper's claims do not require a novel concentration theorem; they require recognizing that any such theorem supplies the exact same false-positive budget in Theorem~\ref{thm:finite-audit}.

\subsection{Boundary constructions}

\begin{proposition}[Clipped-cycle construction]
Let $a,b\in\Theta$ with $\E(a)<\E(b)$, and define $\theta_t=a$ for even $t$ and $\theta_t=b$ for odd $t$. Then signed cumulative progress over every even horizon is zero, but clipped cumulative progress grows linearly in $T$.
\end{proposition}

\begin{proof}
Each two-step cycle has signed reward $(\E(a)-\E(b))+(\E(b)-\E(a))=0$. The clipped reward for the $a\to b$ step is $0$ and for the $b\to a$ step is $\E(b)-\E(a)>0$. Thus every two-step cycle pays positive clipped reward at zero net endpoint improvement.
\end{proof}

\begin{proposition}[Stream-audit separation]
There exist losses $\E_{\Pstream_t}$ and $\E_\Q$ and a trajectory $\theta_t$ such that $\E_{\Pstream_t}(\theta_{t-1})-\E_{\Pstream_t}(\theta_t)>0$ for all $t$ while $\E_\Q(\theta_T)\ge \E_\Q(\theta_0)$.
\end{proposition}

\begin{proof}
Take a two-coordinate linear prediction problem. Let the audit loss depend only on coordinate 1 and the stream loss at time $t$ depend only on coordinate 2. A policy that improves coordinate 2 while degrading coordinate 1 has positive stream progress and non-improving audit loss. The stream-versus-audit comparison instantiates this separation in the ARC-style setting.
\end{proof}

\begin{proposition}[Reusable-panel memorization]
If a model class can interpolate arbitrary labels on the finite audit panel while assigning poor probabilities off-panel, then empirical audit loss can be driven to zero while population audit loss remains high.
\end{proposition}

\begin{proof}
Choose a model that memorizes each panel point and behaves as a bad predictor elsewhere. Since the panel has measure zero under a continuous population distribution, empirical loss can vanish while population loss is governed by off-panel behavior. In discrete settings, the same construction holds when the panel is a small subset of the support and the predictor is bad on the complement. The panel-memorization test and the holdout-reuse attack probe the white-box and scalar-feedback forms of this phenomenon.
\end{proof}

\section{Lean declaration excerpts}

\subsection{Telescoping and finite budget}

\begin{lstlisting}[style=leanstyle]
theorem cumCP_telescope {iota : Type*} (E : iota -> Real)
    (g : Nat -> iota) (T : Nat) :
    cumCP E g T = E (g 0) - E (g T) := by
  induction T with
  | zero => simp [cumCP]
  | succ n ih =>
      simp only [cumCP, Finset.sum_range_succ] at ih |- 
      rw [ih]; ring

theorem cumCP_le_of_lb {iota : Type*} (E : iota -> Real)
    (g : Nat -> iota) (T : Nat) {Emin : Real}
    (hlb : Emin <= E (g T)) :
    cumCP E g T <= E (g 0) - Emin := by
  rw [cumCP_telescope]; linarith
\end{lstlisting}

\subsection{Entropy floor}

\begin{lstlisting}[style=leanstyle]
theorem klDivFinitePMF_nonneg {H : Type*} [Fintype H]
    (Q P : FinitePMF H) [HasPositivePrior P] :
    0 <= klDivFinitePMF Q P := by
  ...

-- T3 as stated in the paper: the CONDITIONAL (input-averaged) entropy floor. The floor
-- H(Y|X) <= E_x H(Q, P (g T)) is DISCHARGED by conditional Gibbs
-- (condEntropy_le_condCrossEntropy), not taken as a hypothesis. The unconditional
-- entropy_floor_budget_crossEntropy is the single-input special case.
theorem cond_entropy_floor_budget_crossEntropy {iota X H : Type*} [Fintype X] [Fintype H]
    (mu : FinitePMF X) (Q : X -> FinitePMF H) (P : iota -> X -> FinitePMF H)
    [forall t, forall x, HasPositivePrior (P t x)] (g : Nat -> iota) (T : Nat) :
    cumCP (fun t => condCrossEntropyFinitePMF mu Q (P t)) g T
      <= condCrossEntropyFinitePMF mu Q (P (g 0)) - condEntropyFinitePMF mu Q :=
  cumCP_le_of_lb _ g T (condEntropy_le_condCrossEntropy mu Q (P (g T)))
\end{lstlisting}

\section{Experiment protocol summary}

\paragraph{Reward-signal ablation.}
The task source consists of 24 learnable ARC-TGI grid-to-grid generator families and 8 i.i.d. uniform distractor families. The model is trained for 5000 steps. The scheduler is held fixed across reward variants. Reward variants are signed audit-CP, prediction error, RND, ICM, uniform, round-robin, and learnable-only oracle. Metrics are active-cell audit accuracy, floored audit cross-entropy, and distractor allocation fraction.

\paragraph{Finite-audit concentration.}
Sample finite families of predictors and audit panels at several $n$. Compute empirical uniform deviation between panel loss and a larger fresh-audit estimate. Fit a log-log slope of deviation against $n$.

\paragraph{Scalar-feedback attack.}
An adaptive attacker receives scalar audit feedback from a reusable panel and selects candidate updates. We compare naive reusable release against fresh subsampling, ladder release, rounded release, and one-shot release. The reported gap is reusable CP minus fresh CP; success is crossing $2\Delta_n$.

\paragraph{Boundary tests.}
The stream-versus-audit comparison contrasts stream-scored CP with sealed-audit CP under increasing stream bias. The clipping comparison cycles between states to compare signed and clipped progress. The panel-memorization test trains directly on a reusable panel to probe high-capacity memorization. The calibration probe applies temperature scaling and checks that hard accuracy is invariant while NLL/ECE move.

\end{document}